\documentclass{article}

\usepackage{microtype}
\usepackage{graphicx}
\usepackage{subfigure}
\usepackage{booktabs} 

\usepackage{hyperref}

\usepackage{iclr2021_conference,times}
\usepackage[utf8]{inputenc} 
\usepackage[T1]{fontenc}    
\usepackage{url}          
\usepackage{booktabs}       
\usepackage{amsfonts}       
\usepackage{nicefrac}       
\usepackage{microtype}      
\usepackage{comment}
\usepackage{amsmath}
\usepackage[english]{babel}
\usepackage{epsfig}
\usepackage{graphicx}
\usepackage{wrapfig}
\usepackage{subfigure}
\usepackage{multirow}
\usepackage{graphicx}
\usepackage{array}
\usepackage{bbm}
\usepackage{color}
\usepackage{enumerate}
\usepackage{enumitem}
\setlist{leftmargin=10mm}
\usepackage{mathtools}
\usepackage{amsmath,amssymb,amsthm,bm}
\usepackage{amsfonts,graphicx}
\usepackage{mathrsfs}
\usepackage{algorithmic,algorithm}

\usepackage{natbib}

\DeclareMathOperator*{\argmax}{argmax}

\newcommand{\blue}[1]{\textcolor{blue}{#1}}

\def\pr{\mathrm{Pr}}
\def\E{\mathbb{E}}
\def\P{\mathbb{P}}

\def \E{\mathbb{E}}

\def\R{\mathbb{R}}

\def\cD{\mathcal{D}}

\def\cM{\mathcal{M}}
\def\cN{\mathcal{N}}

\def\cR{\mathcal{R}}

\def\cX{\mathcal{X}}
\def\cY{\mathcal{Y}}

\newtheorem{theorem}{Theorem}
\newtheorem{lemma}[theorem]{Lemma}

\newtheorem{definition}[theorem]{Definition}
\newtheorem{example}[theorem]{Example}
\theoremstyle{definition}

\newtheorem{remark-star}{Remark}
\newtheorem{remark-star-1}{Remark}

\newtheorem{proposition}[theorem]{Proposition}
\newtheorem*{proof-sketch}{Proof Sketch}
\usepackage{booktabs}

\title{Voting-based Approaches for Differentially Private Federated Learning}

\iclrfinalcopy

\author{Yuqing Zhu$^{1,2}$, Xiang Yu$^{2}$, Yi-Hsuan Tsai$^{2}$, Francesco Pittaluga$^{2}$, Masoud Faraki$^{2}$, \\
	\textbf{Manmohan chandraker$^{2,3}$} and \textbf{Yu-Xiang Wang$^{1}$} \\
	$^{1}$University of California, Santa Barbara\\
	$^{2}$NEC Laboratories America \\
	$^{3}$University of California, San Diego\\
	\texttt{\{yuqingzhu,yuxiangw\}@ucsb.edu} \\ \texttt{\{xiangyu,ytsai,francescopittaluga,mfaraki,manu\}@nec-labs.com} \\
}

\begin{document}
	\maketitle
\begin{abstract}
	Differentially Private Federated Learning (DPFL) is an emerging field with many applications. Gradient averaging based DPFL methods require costly communication rounds and hardly work with large-capacity models, due to the explicit dimension dependence in its added noise. In this work, inspired by knowledge transfer non-federated privacy learning from Papernot et al.(2017; 2018), we design two new DPFL schemes, by voting among the data \emph{labels} returned from each local model, instead of averaging the gradients, which avoids the dimension dependence and significantly reduces the communication cost. Theoretically, by applying secure multi-party computation, we could exponentially amplify the (data-dependent) privacy guarantees when the margin of the voting scores are large. Extensive experiments show that our approaches significantly improve the privacy-utility trade-off over the state-of-the-arts in DPFL.
\end{abstract}

\vspace{-0.3em}
\section{Introduction}
\vspace{-0.3em}
Federated learning (FL) \citep{fedavg,mpc,mohassel2017secureml,smith2017federated}  is an emerging paradigm of distributed machine learning with a wide range of applications \citep{kairouz2019advances}. FL allows distributed agents to collaboratively train a centralized machine learning model without sharing each of their local data, thereby sidestepping the ethical and legal concerns that arise in collecting private user data for the purpose of building machine-learning based product and services.

The workflow of FL is often enhanced by the secure multi-party computation \citep{mpc} (MPC), so as to handle various threat models in the communication protocols, which provably ensures that the agents could receive the output of the computation (e.g.,  the sum of the gradients) but nothing in between (e.g., other agents' gradient). 

However, MPC alone does not protect the agents or their users from inference attacks that use only the output, or combine the output with auxiliary information. Extensive studies demonstrate that these attacks could lead to a blatant reconstruction of the proprietary datasets \citep{dinur2003revealing}, high-confidence identification of individuals (a legal liability for the participating agents)~\citep{shokri2017membership}, or even completion of social security numbers \citep{carlini2019secret}. Motivated by these challenges, there had been a number of recent efforts \citep{dpsgd2019, dp_fl, dp_language} in developing federated learning methods with differential privacy (DP) \citep{dwork2006calibrating}, a well-established definition of privacy that provably prevents such attacks.

Existing methods in differentially private federated learning (DPFL), e.g., DP-FedAvg \citep{dp_fl, dp_language}  and the recent state-of-the-art DP-FedSGD \citep{dpsgd2019}, are predominantly noisy gradient based methods, which build upon 
the NoisySGD method, a classical algorithm in (non-federated) DP learning \citep{song2013stochastic,bassily2014private,abadi2016deep}. They work by iteratively aggregating (multi-)gradient updates from individual agents using a differentially private mechanism.
A notable limitation for this approach is that they require clipping the $\ell_2$ magnitude of gradients to $S$ and adding noise proportional to $S$ with \emph{every coordinate} of the high dimensional parameters from the shared global model. The clipping and perturbation steps introduce either large bias (when $S$ is small) or large variance (when $S$ is large), which interferes the SGD convergence and makes it hard to scale up to large-capacity models. In Sec. \ref{sec:challenge}, we concretely demonstrate these limitations with examples and theory. Particularly, we show that the FedAvg may fail to decrease the loss function using gradient clipping, and DP-FedAvg requires many outer-loop iterations (i.e., many rounds of communication to synchronize model parameters) to converge under differential privacy. 

In this paper, we consider a fundamentally different DP learning setting known as the \emph{Knowledge Transfer} model \citep{papernot2017} (a.k.a. the \emph{Model-Agnostic Private learning} model \citep{bassily2018model}). This model requires an \emph{unlabeled} dataset to be available \emph{in the clear}, which makes this setting slightly more restrictive. However, when such a public dataset is indeed available (it often is in federated learning with domain adaptation, see, e.g., ~\citet{peterson2019private,agnostic_fl,peng2019}),  it could substantially improve the privacy-utility tradeoff in DP learning  \citep{papernot2017, papernot2018, private_kNN}.  

The modest goal of this paper is to develop DPFL algorithms under the \emph{knowledge transfer} model, for which we propose two algorithms (\textit{AE-DPFL} and \textit{kNN-DPFL} ), that further develop from the \textit{non-distributed}  Private-Aggregation-of-Teacher-Ensembles (PATE) \citep{papernot2018} and Private-kNN \citep{private_kNN} to the FL setting. We discover that the distinctive characteristics of these algorithms make them \emph{natural} and \emph{highly desirable} for DPFL tasks. Specifically, the private aggregation is now essentially privately releasing ``ballot counts'' in the (one-hot) label space, instead of the parameter (gradient) space. This naturally avoids the aforementioned issues associated with high dimensionality and gradient clipping. Instead of transmitting the gradient update, transmitting the vote of the ``ballot counts'' tremendously reduce the communication cost. Moreover, many iterations of the model update using noise addition with SGD, leads to poor privacy guarantee, where our methods exactly avoid this and use voting on labels, thus significantly outperform the state-of-the-art DPFL methods.

Our contributions are summarized in four folds.
\begin{enumerate}[noitemsep, nolistsep]
	\item{We construct  examples to demonstrate that DP-FedAvg (a) may fail due to gradient clipping and (b) requires many rounds of communications; while our approach naturally avoids both limitations.}
	\item{We identify the two typical regimes FL operates in and highlight that they demand two distinct granularity levels of DP protections: \emph{agent-level} or \emph{instance (of-each-agent)-level}. }
	\item{We design distributed algorithm that provides provable DP guarantees on both \emph{agent-level} and \emph{instance-level} against strong adversaries.  By a new MPC technique \citep{dery2019fear}, we show that our proposed private voting mechanism enjoys an \emph{exponentially stronger} (data-dependent) privacy guarantee when the ``winner'' wins by a large margin.}
	\item{Extensive evaluation demonstrates that our method systematically improves the privacy-utility trade-off over DP-FedAvg and DP-FedSGD, and that our methods are more robust towards distribution-shifts across agents. }
\end{enumerate}

\noindent \textbf{A remark of our novelty.}
Though \textit{AE-DPFL} and \textit{kNN-DPFL} are algorithmically similar to the original \textit{PATE}~\citep{papernot2018} and \textit{Private-KNN}~\citep{private_kNN}, they are not the same and we facilitate them to a new problem --- \emph{federated learning}. The facilitation itself is nontrivial and requires substantial technical innovations. We highlight three challenges below.

To begin with, several key DP techniques that contribute to the success of PATE and Private-kNN in the standard settings are no longer applicable (e.g., privacy amplification by sampling and noisy screening). This is partly because in standard private learning, the attacker only sees the final models; but in FL, the attacker can eavesdrop in all network traffic and could be a subset of the agents themselves.  

Moreover, PATE and Private-kNN only provide instance-level DP. We show \textit{AE-DPFL} and \textit{kNN-DPFL} also satisfy the stronger agent-level DP.  \textit{AE-DPFL}'s agent-level DP parameter is, surprisingly, a factor of $2$ better than its instance-level DP parameter. \textit{kNN-DPFL} in addition enjoys a factor of $k$ amplification for the instance-level DP.

Finally, a key challenge of FL is the data heterogeneity of individual agents, where PATE randomly splits the dataset so each teacher is identically distributed. The heterogeneity 
violates the key I.I.D assumption which PATE \citep{bassily2018model,liu2020revisiting} relies on.
With our novel design, we are the first to report strong empirical evidence that the PATE-like DP algorithms can still remain highly effective in the non-I.I.D setting. 

\vspace{-0.5em}
\section{Preliminary}
\vspace{-0.5em}
\begin{figure*}[t]
	\centering	
	\includegraphics[width=0.95\linewidth]{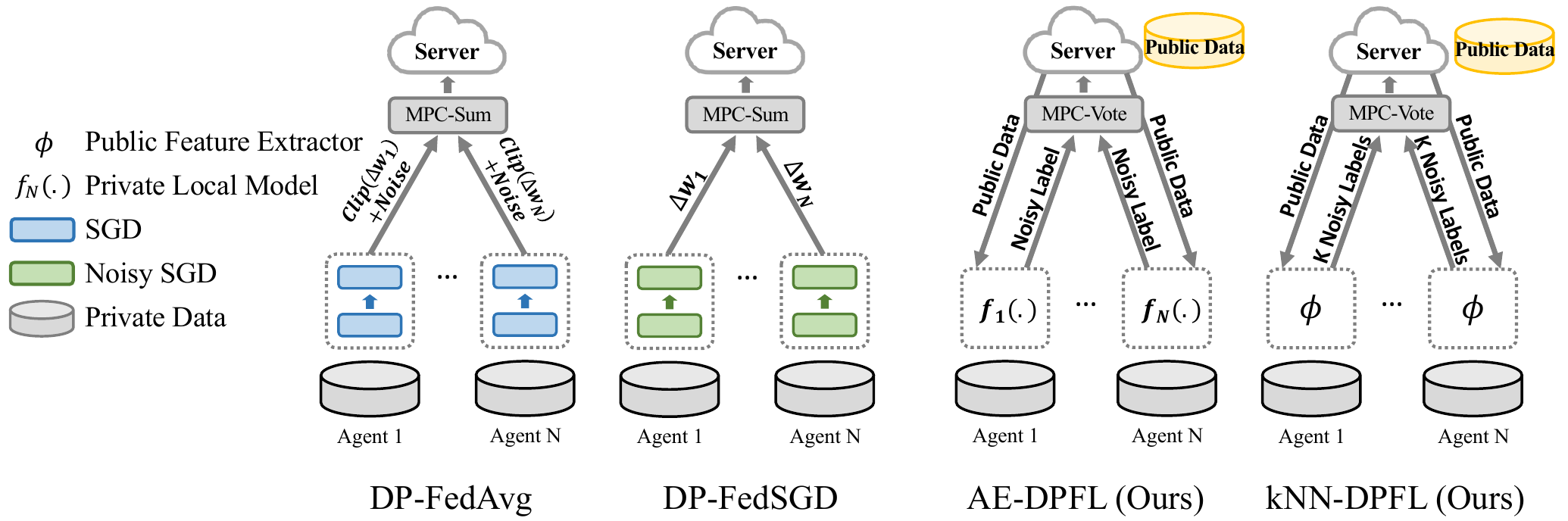}
	\vspace{-3mm}
	\caption{The structural difference of our methods to DP-FedAvg and DP-FedSGD. DP-FedAvg and \textit{AE-DPFL} are for agent-level DP. DP-FedSGD and \textit{kNN-DPFL} are for instance-level DP.}
	\label{fig:illus}
	\vspace{-5mm}
\end{figure*}
In this section, we start by introducing the typical notations of federated learning and differential privacy. Then, by introducing the two different level DP definitions, two randomized gradient-based baselines, DP-FedAvg and DP-FedSGD, are introduced as the DPFL background.

\vspace{-0.5em}
\subsection{Federated Learning}
\vspace{-0.5em}
\noindent\textbf{Notation} We consider $N$ agents, each agent $i$ has $n_i$ data kept local and private from a party-specific domain distribution $\cD_i \in \cX \times \cY$, where $\cX$ denotes the feature space and $\cY =\{0, ..., C-1\}$ denotes the label.

\noindent\textbf{Problem setting} The goal is to train a privacy-preserving global model that performs well on the server distribution $\cD_G$ without centralizing local agent data. We assume access to an \emph{unlabeled} dataset containing I.I.D samples from the server distribution $\cD_G$. This is a standard assumption from ``agnostic federated learning'' ~\citep{agnostic_fl, peng2019,peterson2019private} literature, and more flexible than fixing $\cD_G$ to be the uniform user distribution over the union of all agents. The choice of $\cD_G$ is application-specific and it represents the various considerations of the learning objective such as accuracy, fairness and the need for personalization. The setting is closely related to the multi-source domain adaptation problem~\cite{zhang2015multi} but is more challenging due to the restricted access of the source (local) data.





\noindent\textbf{FL Baseline:} FedAvg~\citep{fedavg} is a vanilla federated learning algorithm as a non-DP baseline. A fraction of agents is sampled at each communication round with a probability $q$. Each selected agent downloads the shared global model and fine-tune with local data for $E$ iterations using stochastic gradient descent (SGD). We denote this local update process as an \textit{inner loop}. Then, only the gradients are sent to the server, and averaged across all the selected agents to improve the global model. The global model is learned after $T$ communication rounds. Such communication round is denoted as one \textit{outer loop}. 
\vspace{-0.5em}
\subsection{Differential Privacy for Federated Learning}

Differential privacy~\citep{dwork2006calibrating} is a quantifiable definition of privacy that provides provable guarantees against identification of individuals in a private dataset. 
\begin{definition}
	\textbf{Differential Privacy:} A randomized mechanism $\cM: \cD \to \cR$ with a domain $\cD$ and range $\cR$ satisfies $(\epsilon, \delta)$-differential privacy, if for any two \emph{adjacent} datasets $D, D' \in \cD$ and for any subset of outputs $O  \subseteq \cR$, it holds that $\pr[\cM(D) \in O]\leq e^\epsilon \pr[\cM(D')\in O ] + \delta$.
\end{definition}

The definition indicates that one could not distinguish between $D$ and $D'$ therefore protecting the ``delta'' between $D,D'$. Depending on how \emph{adjacency} is defined, this ``delta'' comes with different semantic meaning. We consider two levels of granularity:
\begin{definition}
	\textbf{Agent-level DP:} When $D'$ is constructed by adding or removing an agent from $D$ (with all data points from that agent).
\end{definition}
\begin{definition}
	\textbf{Instance-level DP:} When $D'$ is constructed by adding or removing one data point from any of the agents.
\end{definition}
Among the applications of DPFL, the above two definitions have their own favorable situations. For example, when a smart phone app jointly learns from its users' text messages, it is more appropriate to protect each user as a unit, which is agent-level DP. In another situation, when five hospitals would like to collaborate on a patient study via federated learning, obfuscating the entire dataset from one hospital is meaningless, where the instance-level DP is more suitable to protect the individual patient from being identified.

\noindent\textbf{DPFL Baselines:} DP-FedAvg~\citep{dp_fl, dp_language} (Figure~\ref{fig:illus} and Algorithm~\ref{alg:dp_fedavg}), a representative DPFL algorithm, when compared to FedAvg, it enforces clipping of per-agent model gradient to a threshold $S$ (see Step $3$ in Algorithm~\ref{alg:dp_fedavg} NoisyUpdate) and adds noise to the scaled gradient before it is averaged at the server, which ensures agent-level DP. 
DP-FedSGD \citep{dpsgd2019,peterson2019private}, is one of the state-of-the-arts that focus on instance-level DP. It performs NoisySGD~\citep{abadi2016deep} for a fixed number of iterations at each agent. The gradient updates are averaged on each communication round at the server, as shown in Figure~\ref{fig:illus}. 



\noindent\textbf{Multi-Party Computation (MPC):} 
MPC is a cryptographic technique that securely aggregates local updates before the server receives it. While MPC does not have a differential privacy guarantee, it can be combined with DP to \emph{amplify} the privacy guarantee \citep{mpcapplication,cpsgd,dpsgd2019}.
Specifically, if each party adds a small independent noise to the part they contribute, MPC ensures that an attacker can only observe the total, even if he taps the network messages and hacks into the server. In this paper, we consider a new MPC technique due to \citep{dery2019fear} that allows only the voted winner to be released while keeping the voting scores completely hidden. This allows us to further amplify the DP guarantees. In our experiment, we assume the aggregation is conducted by MPC for all privacy-preserving algorithms that we consider (see Figure~\ref{fig:illus}).


\begin{algorithm}[t]
	\caption{DP-FedAvg  }
	\label{alg:dp_fedavg}
	{\bf Input:}
	Agent selection probability $q$, noise scale $\sigma$, clipping threshold $S$.
	\begin{algorithmic}[1]
		\STATE{Initialize global model $\theta^0$}
		\STATE \textbf{for}  $t=0, 1,2,..., T$ \textbf{do}
		\STATE  \ \ \ \ \ \ $m_t \gets$ Sample agents with  $q$
		\STATE \ \ \ \ \ \ \textbf{for} each agent $i$ in parallel \textbf{do}
		\STATE \ \ \ \  \ \  \ \ \ \ \ \ $\triangle_i^{t} =$ NoisyUpdate($i, \theta^t, t, \sigma, m_t$)
		\STATE \ \ \ \ \ \ \ $\theta^{t+1} = \theta^t + \frac{1}{m_t}\sum_{i=0}^{m_t} {\triangle_i}^t $
	\end{algorithmic}
	
	NoisyUpdate $(i, \theta^0, t, \sigma, m_t)$
	\begin{algorithmic}[1]
		\STATE $\theta \gets \theta^0$
		\STATE $\theta \gets E$ iterations SGD from $\theta^0$
		\STATE  ${\triangle_i}^t= (\theta - \theta^0)/\max(1, \frac{|| \theta - \theta^0||_2}{S})$
		\STATE return update ${\triangle_i}^t +\cN(0, \sigma^2S^2/{m_t})$				
	\end{algorithmic}
	
\end{algorithm}

\section{Challenges in Gradient-based FL}\label{sec:challenge}
\vspace{-0.5em}
In this section, before introducing our approaches, we motivate by highlighting the main challenges in the conventional DPFL methods in terms of gradient estimation, convergence, and data heterogeneity. For other challenges, we refer the readers to a survey~\citep{kairouz2019advances}. 

\noindent\textbf{Challenge 1: Biased Gradient Estimation.}  
Recent works~\citep{li2018federated} have shown that the FedAvg may not converge well under data heterogeneity. We provide a simple example to show that the clipping step of DP-FedAvg may exacerbate the issue. 

\begin{example}
	Let $N=2$, each agent $i$'s local update is $\triangle_i$ ($E$ iterations of SGD). We enforce clipping of per-agent update $\triangle_i$ by performing $\triangle_i / \max(1, \frac{||\triangle_i||_2}{S})$, where $S$ is the clipping threshold.  Consider the special case when $||\triangle_1||_2 = S +\alpha$ and $||\triangle_2||_2 \leq S$. Then the global update will be $\frac{1}{2}(\frac{S \triangle_1}{||\triangle_1||_2}+\triangle_2)$, which is biased.
\end{example}
Comparing to the FedAvg updates $\frac{1}{2}(\triangle_1 + \triangle_2)$, the biased update could be $0$ (not moving) or pointing towards the opposite direction.
Such a simple example can be embedded in more realistic problems, causing substantial bias that leads to non-convergence.

\noindent\textbf{Challenge 2: Slow Convergence.}
Following works on FL convergence analysis~\citep{li2019convergence,wang2019adaptive}, we derive the convergence analysis on DP-FedAvg and demonstrate that using many outer-loop iterations ($T$) could result in similar convergence issue under differential privacy.



The appeal of FedAvg is to set $E$ to be larger so that each agent performs $E$ iterations to update its own parameters before synchronizing the parameters to the global model, hence reducing the number of rounds in communication. We show that the effect of increasing $E$ is essentially increasing the learning rate for a large family of optimization problems with piece-wise linear objective functions, which does not change the convergence rate. The detailed analysis is in appendix convergence section due to space limit. Specifically, it is known that for the family of $G-$Lipschitz functions supported on a $B$-bounded domain, any Krylov-space method \footnote{One that outputs a solution in the subspace spanned by a sequence of sub-gradients.} has convergence rate that is lower bounded by $\Omega(BG/\sqrt{T})$ \citep[Section 3.2.1]{nesterov2003introductory}. This indicates that the variant of FedAvg requires $\Omega(1/\alpha^2)$ rounds of outer loop (i.e., communication), in order to converge to an $\alpha$ stationary point, i.e., increasing $E$ does \emph{not} help, even if no noise is added.

It also indicates that DP-FedAvg is essentially the same as \emph{stochastic} sub-gradient method in almost all locations of a piece-wise linear objective function with gradient noise being $\cN(0, \sigma^2/{N} I_d)$. The additional noise in DP-FedAvg imposes more challenges to the convergence. If we plan to run $T$ rounds and achieve $(\epsilon,\delta)$-DP, we need to choose $\sigma = \frac{\eta EG\sqrt{2T\log(1.25/\delta)}}{N\epsilon}$ \citep[see, e.g., ][Theorem~1]{dp_language}, 
which results in a convergence rate upper bound of 
\vspace{-0.3em}
{\small
	$$
	\frac{GB (\sqrt{1 +  \frac{2Td\log(1.25/\delta)}{N^2\epsilon^2}} )}{\sqrt{T}}  = O\left( \frac{GB}{\sqrt{T}} + \frac{\sqrt{d\log(1.25/\delta)}}{N\epsilon}\right),
	$$
}
for an optimal choice of the learning rate $E\eta$.

The above bound is tight for stochastic sub-gradient methods, and in fact also information-theoretically optimal. The $GB/\sqrt{T}$ part of upper bound matches the information-theoretical lower bound for all methods that have access to $T$-calls of stochastic sub-gradient oracle \citep[Theorem 1]{agarwal2009information}. While the second matches the information-theoretical lower bound for all $(\epsilon,\delta)$-differentially private methods on the agent level \citep[Theorem 5.3]{bassily2014private}. That is, the first term indicates that there must be \emph{many rounds of communications}, while the second term says that the \emph{dependence in ambient dimension $d$} is unavoidable for DP-FedAvg. Clearly, our method also has such dependence \emph{in the worst case}. But it is easier for our approach to adapt to the structure that exists in the data (i.e., high consensus among voting), as we will illustrate later. In contrast, it has larger impact on DP-FedAvg, since it needs to explicitly add noise with variance $\Omega(d)$.

Another observation is when $N$ is small, no DP method with reasonable $\epsilon,\delta$ parameters is able to achieve high accuracy for agent-level DP. This partially motivates us to consider the other regime that deals with instance-level DP.

\textbf{Challenge 3: Data Heterogeneity.} 
Federated learning with domain adaptation has been studied in~\citet{peng2019}, where they propose a dynamic attention model to adjust the contribution from each source (agent) collaboratively. However, most multi-source domain adaptation algorithms, including this approach, require sharing local feature vectors to the target domain, which is not compatible with the DP setting. Enhancing \textit{DP-FedAvg} with the effective domain adaptation technique remains an open problem.

\section{Our Approach}
\vspace{-0.5em}
To alleviate above challenges, we propose two voting-base algorithms, termed aggregation ensemble DPFL ``\textit{AE-DPFL}'' and k Nearest Neighbor DPFL ``\textit{kNN-DPFL}''.  Each algorithm first privately labels a subset of data from the server and then trains a global model using pseudo-labeled data.
\vspace{-1em}
\subsection{Aggregation Ensemble - DPFL }
\vspace{-0.5em}


In \textit{AE-DPFL} (Algorithm~\ref{alg:pate}), each agent $i$ trains a local agent model $f_i$ using its own private local data. The local model is never revealed to the server but only used to make predictions for unlabeled data (queries). For each query $x_t$, every agent $i$ adds Gaussian Noise to the prediction (i.e., $C$-dimensional histogram where each bin is zero except the $f_i(x_t)$-th bin is $1$). The ``pseudo label'' is achieved with the majority vote returned by aggregating the noisy predictions from the local agents.
%


For instance-level DP, the spirit of our method shares with PATE, in the aspect of by adding or removing one instance, it can \textit{change} at most one agent's prediction. The same argument also naturally applies to \emph{adding or removing one agent}. In fact we gain a factor of $2$ in the stronger agent-level DP due to a smaller sensitivity in our approach (see proof in appendix section C). 
Another important difference is that in the original PATE, the teacher models are are trained on I.I.D data (random splits of the whole private data), while in our case, the agents are naturally present with different distributions. We propose to optionally use domain adaptation techniques to mitigate these differences when training the agents.


\vspace{-1em}
\subsection{kNN - DPFL}
\vspace{-0.5em}



From Definition 2 and 3, preserving agent-level DP is generally more difficult than the instance-level DP. We find that for \textit{AE-DPFL}, the privacy guarantee for instance-level DP is weaker than its agent-level DP guarantee (see Theorem~\ref{thm:privacy}). To amplify the instance-level DP, we now introduce our \textit{kNN-DPFL}.

As in Algorithm~\ref{alg:knn}, each agent maintains a data-independent feature extractor $\phi$, i.e., an ImageNet~\citep{deng2009imagenet} pre-trained network without the classifier layer. For each unlabeled query $x_t$, agent $i$ first finds the $k_i$ nearest neighbors to $x_t$ from its local data by measuring the Euclidean distance in the feature space $\cR^{d_\phi}$. Then, $f_i(x_t)$ outputs the frequency vector of the votes from the nearest neighbors, which equals to $\frac{1}{k}(\sum_{j=1}^k y_j)$, where $y_j \in \cR^C$ indicates the one-hot vector of the ground-truth label. Subsequently, $\tilde{f}_i(x_t)$ from all agents are privately aggregated with the argmax of the noisy voting scores returned to the server.
%

Different from Private-kNN, we apply kNN
on each agent's local data instead of the entire private dataset. This distinction allows us to receive up to $kN$ neighbors while bounding the contribution of individual agents by $k$. Comparing to \textit{AE-DPFL}, this approach enjoys a stronger instance-level DP guarantee since the sensitivity from adding or removing one instance is a factor of $k/2$ times smaller than that of the agent-level.



\begin{figure}[t]
	\centering
	\begin{minipage}{0.47\textwidth}
		\begin{algorithm}[H]
			\caption{\textit{AE-DPFL}  with MPC-Vote}\label{alg:pate}
			\begin{algorithmic}[1]
				\INPUT{Noise level $\sigma$,  unlabeled public data $\cD_G$, integer $Q$.}
				\STATE{ Train local model $f_i$ using $\cD_i$}
				\FOR{$t=0, 1,..., Q$, pick $x_t\in \cD_G$ }
				\FOR{each agent  $i$ in $1, ..., N$ (in parallel)}
				\STATE $\tilde{f_i}(x_t) = f_i(x_t) + \cN(0, \frac{\sigma^2}{N} I_C)$. 
				\ENDFOR
				\STATE  $\tilde{y}_t = \argmax_{y \in \{1,...,C\}} [\sum_{i=1}^{N} \tilde{f}_i(x_t)]_y$ via MPC.
				\ENDFOR
				\OUTPUT  A global model $\theta$ trained using ${(x_t, \tilde{y}_t)}_{t=1}^{Q}$
			\end{algorithmic}
		\end{algorithm}
	\end{minipage}
	\quad
	\begin{minipage}{0.47\textwidth}
		\begin{algorithm}[H]
			\caption{\textit{kNN-DPFL} with MPC-Vote}
			\label{alg:knn}
			\begin{algorithmic}[1]
				\INPUT{Noise level $\sigma$, unlabeled public data $\cD_G$, integer $Q$, feature map $\phi$.}
				\FOR{$t=0, 1,..., Q$, pick $x_t\in \cD_G$ }
				\FOR{each agent  $i$ in $1, ..., N$ (in parallel)}
				\STATE{ Apply  $\phi$ on $\cD_i$ and $x_t$}
				\STATE{ $y_1, ..., y_k \gets$ labels of the k nearest neighbor. }
				\STATE{ $\tilde{f}_i(x_t) = \frac{1}{k}(\sum_{j=1}^k y_j)+ \cN(0,\frac{\sigma^2}{N} I_C)$}
				\ENDFOR
				\STATE $\tilde{y}_t = \argmax_{y \in \{1,...,C\}} [\sum_{i=1}^{N} \tilde{f}_i(x_t)]_y$ via MPC.
				\ENDFOR
				\OUTPUT A global model $\theta$ trained using ${(x_t, \tilde{y}_t)}_{t=1}^{Q}$
			\end{algorithmic}
		\end{algorithm}
	\end{minipage}
	\vspace{-1em}
\end{figure}
\vspace{-1em}
\subsection{Privacy Analysis}
\vspace{-0.5em}
Our privacy analysis is based on Renyi differential privacy (RDP)~\citep{mironov2017renyi}. RDP inherits and generalizes the information-theoretical properties of DP, and has been used for privacy analysis in DP-FedAvg and DP-FedSGD. Notably RDP composes naturally and implies the standard $(\epsilon,\delta)$-DP for all $\delta > 0$.  
We defer the background about RDP, its connection to DP and all proofs of our technical results to the appendix RDP section.

\begin{theorem}[Privacy guarantee]\label{thm:privacy}
	Let \textit{AE-DPFL} and \textit{kNN-DPFL} answer $Q$ queries with noise scale $\sigma$. For agent-level protection, both algorithms guarantee $(\alpha, \frac{Q\alpha}{2\sigma^2})$-RDP for all $\alpha\geq 1$. 
	For instance-level protection, \textit{AE-DPFL} and \textit{kNN-DPFL} obey $ (\alpha, \frac{Q\alpha}{\sigma^2})$ and
	$ (\alpha, \frac{Q\alpha}{k\sigma^2})$-RDP respectively.
\end{theorem}

Theorem 5 suggests that both algorithms achieve agent-level and instance-level differential privacy.  With the same noise injection to the agent's output, \textit{kNN-DPFL} enjoys a \emph{stronger} instance-level DP (by a factor of $k/2$) compared to its agent-level guarantee, while \textit{AE-DPFL}'s instance-level DP is \emph{weaker} by a factor of $2$. Since \textit{AE-DPFL} allows an easy-extension with the domain adaptation technique, we choose to use \textit{AE-DPFL} for the agent-level DP and apply \textit{kNN-DPFL} for the instance-level DP in the experiments.


\textbf{Improved accuracy and privacy with  large margin:}
Let $f_1,...,f_N: \cX\rightarrow \triangle^{C-1}$ where $\triangle^{C-1}$ denotes the probability simplex --- the soft-label space.  Note that both algorithms we propose can be viewed as voting of these local agents, which output a probability distribution in $\triangle^{C-1}$.
First, let us define the margin parameter $\gamma(x)$ that measures the difference between the largest and second largest coordinate of $\frac{1}{N}\sum_{i=1}^N f_i(x)$. 

\begin{lemma}\label{lem: majorite_vote} 
	Conditioning on the local agents, for each server data point $x$,  the noise added to each coordinate of $\frac{1}{N}\sum_{i=1}^N f_i(x)$ is drawn from $\cN(0, \sigma^2/N^2)$, 
	then with probability $\geq 1-C\exp\{-N^2\gamma(x)^2/8\sigma^2\}$, the  privately released label matches the majority vote without noise.
\end{lemma}

The proof (in Appendix C) is a straightforward application of Gaussian tail bounds and a union bound over $C$ coordinates. This lemma implies that for all public data points $x$ such that $\gamma(x) \geq \frac{2\sqrt{2\log(C/\delta)}}{N}$, the output label matches noiseless majority votes with probability at least $1-\delta$.

Next we show that for those data points $x$ such that $\gamma(x)$ is large, the privacy loss for releasing $\argmax_j [\frac{1}{N}\sum_{i=1}^N f_i(x)]_j$ is exponentially smaller. 

\begin{theorem}\label{thm: data_dependent}
	For each public data point $x$, the mechanism that releases $\argmax_j [\frac{1}{N}\sum_{i=1}^N f_i(x) +  \cN(0, (\sigma^2/N^2)  I_C)]_j$ obeys $(\alpha,\epsilon)$-data-dependent-RDP, where
	$$
\epsilon \leq 2Ce^{-\frac{N^2\gamma(x)^2}{8\sigma^2}} +  \frac{1}{\alpha-1}\log\left(1 +   e^{\frac{(2\alpha-1)\alpha s}{2\sigma^2} - \frac{N^2\gamma(x)^2}{8\sigma^2} + \log C/2}\right),
$$
where $s=1$ for AE-DPFL with the agent-level DP, and $s=2/k$ for KNN-DPFL with the instance-level DP.
\end{theorem}
This bound implies when the margin of voting scores is large, the agents enjoy exponentially stronger RDP guarantees in both agent-level and instance-level. In other words, our proposed methods avoid the explicit dependence on model dimension $d$ (unlike DP-FedAvg) and could benefit from ``easy data'' whenever there are high consensus among votes from local agents.
We highlight that Theorem~\ref{thm: data_dependent} is possible because of MPC-vote which ensures that all parties (local agents, server and attackers) observe \emph{only} the $\argmax$ but \emph{not} the noisy-voting scores themselves.

\noindent\textbf{Communication Cost:} Finally, we find that our methods are \emph{embarrassingly parallel} as each agent work independently without any synchronization. Overall, we reduce the (per-agent) up-stream communication cost from $d \cdot T$ floats (model size times $T$ rounds) to $C \cdot Q$, where $C$ is number of classes and $Q$ is the number of data points.

\vspace{-0.5em}
\section{Experimental Results}\label{appendix: experimental}
\vspace{-0.5em}
In this section, we conduct extensive experiments to illustrate the advantages of \textit{AE-DPFL} and \textit{kNN-DPFL} over the conventional DPFLs.  We apply our \textit{AE-DPFL} for agent-level DP and \textit{kNN-DPFL} for instance-level DP based on their distinctive characteristics in privacy guarantee. We organize the following four sets of experiments.
\vspace{-0.5em}
\begin{enumerate}%
	\itemsep0mm
	\item\textbf{Privacy and accuracy trade-off:} 
	We compare privacy or accuracy by aligning the other factor and vice versa in both agent-level and instance-level DP experiments.
	
	\item\textbf{Data heterogeneity:} We investigate various data heterogeneous cases 
	in the agent-level DP setting.
	\item\textbf{Scalability to high-capacity models:} We investigate the performance gap with respect to different network backbones under instance-level DP setting.
	\item\textbf{Data-dependent privacy analysis:} We explore the improved data-dependent privacy (Theorem~\ref{thm: data_dependent}) in the agent-level DP experiments.
\end{enumerate}
\vspace{-0.5em}
The private guarantee of our methods is based on Theorem~\ref{thm:privacy} except the data-dependent privacy $\epsilon^*$ in Table~\ref{tab:result_svhn}. Five independent rounds of experiments are conducted to report mean accuracy and its standard deviation.

\vspace{-1em}
\subsection{Agent-level DP Evaluation}
\vspace{-0.5em}

To investigate various heterogeneous scenarios, we consider three conditions of distributions: (1) data across agents and the server are drawn from different domains (Digit Datasets); (2) I.I.D distribution across agents and the server (CelebA); (3) non-I.I.D partition of local data (MNIST).

Throughout the experiments, we apply ImageNet~\citep{deng2009imagenet} pre-trained AlexNet~\citep{alexnet} as the network backbone.  The hyperparameters of DP-FedAvg include the sampling probability $q$, the noise parameter $\sigma$, $\#$communication rounds $T$ and the clipping threshold $S$. We do a grid search on all hyperparameters, and observe $(S = 0.08, \sigma = 0.06)$ works best under the AlexNet.  The hyperparameters of \textit{AE-DPFL} include the noise scale $\sigma$ and the number of data to be labeled $Q$. For the sake of comparability, we choose to align the privacy cost between two methods and the accuracy we report based on the best models over the choice of the remaining undetermined hyper-parameters (e.g., $T$ in DP-FedAvg) in the cross-validation. More detailed information is referred to Appendix.

\begin{table}
	\centering\label{tab:agent}
		\begin{tabular}{c@{\hskip .1mm}c@{\hskip .01mm}cccc}
			\toprule
			Datasets & \# Agents& Methods & Accuracy ($\%$) &  $\epsilon$   &$\epsilon^*$\\
			\hline
			\multirow{7}{*}{SVHN, MNIST }     &\multirow{7}{*}{200}    & FedAvg         & $87.6\pm0.1$  &-&-   \\
			&    & FedAvg+DA     & $86.9 \pm0.1$ &  - &-\\
			&    &DP-FedAvg     & $76.3\pm0.3$ &  3.7&- \\
			&    & DP-FedAvg+DA     & $71.2\pm0.4$ &  3.6&-\\
			$\to$ USPS   &    & \textit{AE-DPFL} (Ours)  &$83.8\pm 0.2$ &3.6 &2.7 \\
			&    & \textit{AE-DPFL+DA} (Ours)     &$\bf{92.5}\pm\bf{0.2}$ & \bf{2.8}  & \bf{2.1}\\
			\hline
			\multirow{3}{*}{CelebA}   &\multirow{3}{*}{300}       & FedAvg         & $84.9\pm0.1$    &-&- \\
			&   &DP-FedAvg           & $83.2\pm 0.1$&   4.0&-\\
			&   & \textit{AE-DPFL} (Ours)      & $\bf{85.0}\pm \bf{0.1}$&   4.0&-  \\
			\hline
			\multirow{3}{*}{MNIST}   &\multirow{3}{*}{100}       & FedAvg         & $97.8\pm0.1$    &- &-\\
			&   &DP-FedAvg           & $84.2 \pm 0.2$&   4.3&-\\
			&   & \textit{AE-DPFL} (Ours)      & $\bf{95.1}\pm \bf{0.3}$&   4.3 & 4.3\\
			\hline
		\end{tabular}
	\vspace{-2mm}
	
	\caption{\textbf{Agent-level DP Evaluation.} $\epsilon^*$ denotes the data-dependent DP obtained through Theorem~\ref{thm: data_dependent} and we set $\delta=10^{-3}$ for all datasets.}
	\label{tab:result_svhn}
	\vspace{-5mm}
\end{table}

\textbf{Digit Datasets Evaluation}: 
MNIST, SVHN and USPS are put together termed as Digit datasets \citep{mnist,svhn}. It is a controlled setting to mimic the real situations, where distribution of agent-to-server or agent-to-agent can be different. Based on the size of each dataset, we simulate $140$ agents using SVHN with $3000$ records each and $60$ agents using MNIST with $1000$ records each. We split $3000$ unlabeled records from USPS at server and the rest data is used for testing.

We notice that DP-FedAvg and FedAvg never see the server distribution. To boost those two algorithms, we further apply a standard domain adaptation (DA) technique --- adversarial training~\citep{ganin2016domain} on top, denoted as DP-FedAvg+DA and FedAvg+DA, respectively. As a consequence, their local training involves both local data and unlabeled data from the server. Similarly, we define \textit{AE-DPFL+DA} as the DA extension of \textit{AE-DPFL}, where each teacher (agent) model is trained with the same DA technique as that in DP-FedAvg+DA. We set $Q=500$ as the number of data to be labeled for \textit{AE-DPFL} and \textit{AE-DPFL+DA}, and set noise scale $\sigma=25$ for \textit{AE-DPFL} and $\sigma=30$ for \textit{AE-DPFL+DA}. The noise is set larger for \textit{AE-DPFL+DA} because there is a stronger consensus among agent predictions, allowing larger noise level without sacrificing accuracy.

In Table \ref{tab:result_svhn}, we observe: (1) When the privacy cost $\epsilon$ of DP-FedAvg and \textit{AE-DPFL} is close, our method significantly improves the accuracy from 76.3\% to 83.8\%. (2) The further improved accuracy $92.5\%$ of \textit{AE-DPFL+DA} demonstrates that our framework can orthogonally benefit from DA techniques, where it is highly uncertain yet for the gradient-based methods. (3) Both FedAvg and DP-FedAvg perform better than their DA variants; therefore we will only use DP-FedAvg in the following experiments. This result is well expected, as FL with domain adaptation is more closely related to the multi-source domain adaptation~\citep{peng2019moment}. Combining \textit{FedAvg} with the one-source DA methods implies averaging different trajectories towards the server’s distribution, which may not work in practice. Similar learning bound based observation has been investigated in~\citet{peng2019} and it remains unclear how to privatize the multi-source domain adaptation approach.

\textbf{I.I.D CelebA Dataset Evaluation:} 
CelebA~\citep{celeba} is a 220k face attribute dataset with 40 attributes defined. 300 agents are designed with partitioned training data. We split 600 unlabeled data at server, and the rest 59,400 images are for testing. Detailed settings are referred to the Appendix. Consistent to Digits dataset, our method achieves clear performance gain by $1.8\%$ compared to DP-FedAvg while maintaining the same privacy cost.

\begin{figure*}[t]
	\centering	
	\includegraphics[width=0.37\textwidth]{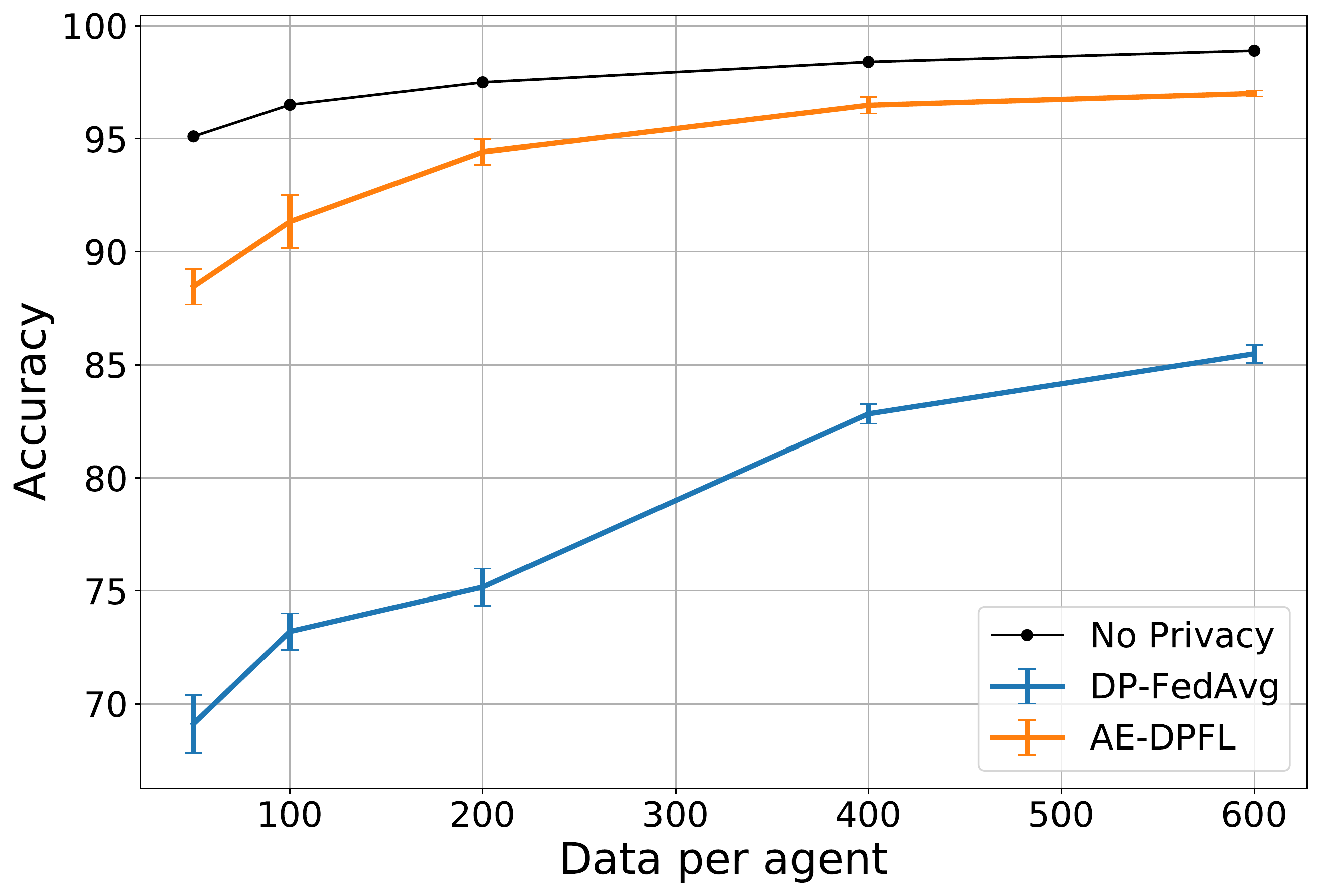}
	\hspace{8mm}
	\includegraphics[width=0.37\textwidth]{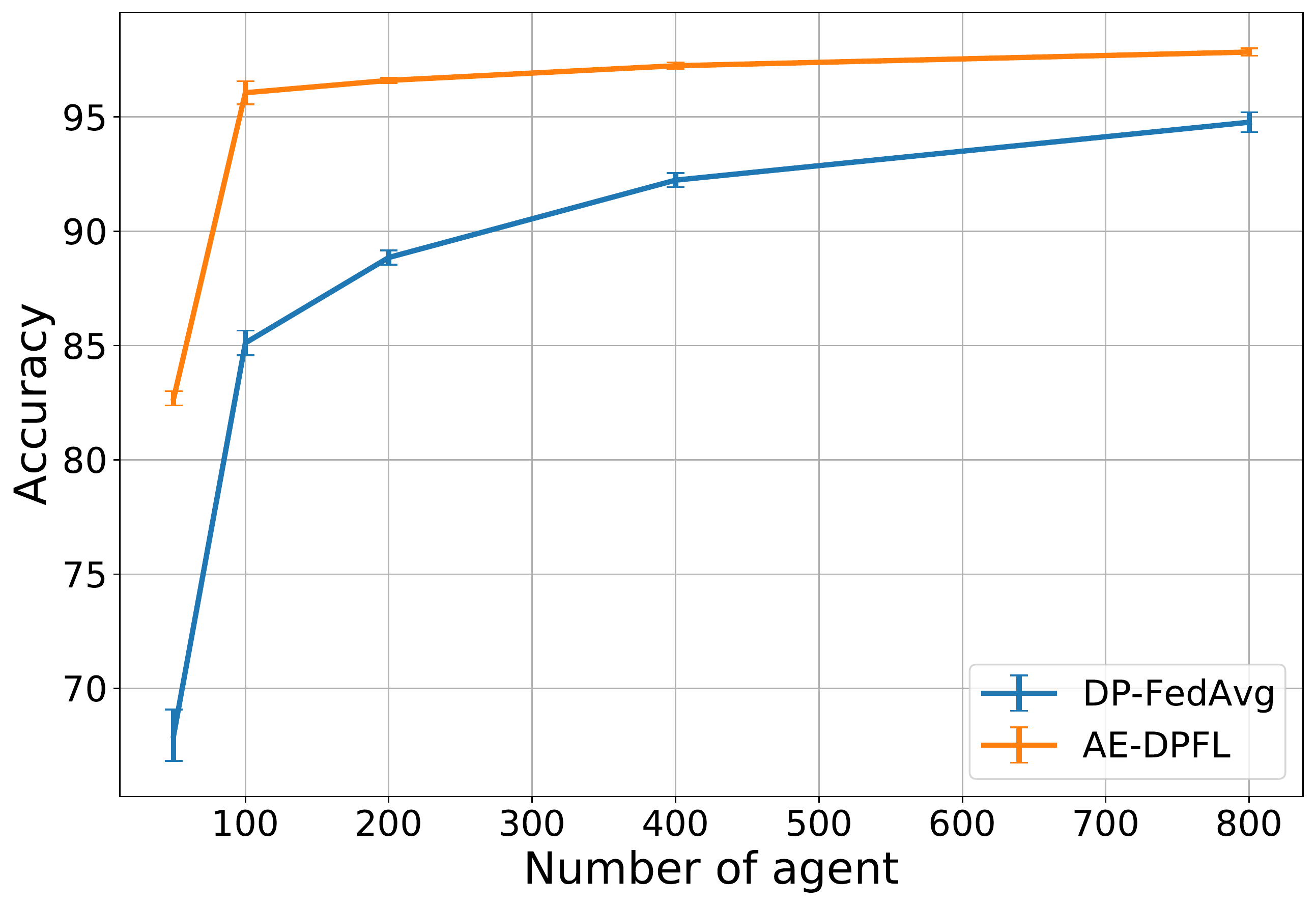}
	\vspace{-3mm}
	\caption{Ablation study. Left: effect on the amount of data per agent. Right: effect on the number of agent. 
	}
	\label{fig:ablation}
	\vspace{-4mm}
\end{figure*}
\textbf{MNIST Dataset with Non-I.I.D Partition:} 
In both CelebA and Digit experiments, we I.I.D partition each dataset into different agents. To investigate our proposed algorithm under a  non-I.I.D partition scenario, we choose a similar experimental setup as \citep{fedavg} did. We divide the training set of the sorted MNIST into $100$ agents, such that each agent will have samples from $6$ digits only. This way, each agent gets $600$ data points from $6$ classes. We split $30\%$ of the testing set in MNIST as the available unlabeled public data and the remaining testing set used for testing. As shown in Table \ref{tab:result_svhn}, our method achieves consistently better performance than DP-FedAvg. 

\begin{table} [t]
	\begin{center}
		\small
		\centerline{
				\tabcolsep=0.001cm
				\begin{tabular}{c@{\hskip .02mm}c@{\hskip .02mm}c@{\hskip .02mm}c@{\hskip .02mm}c@{\hskip .02mm}c@{\hskip .02mm}}
					\toprule
					Network & Methods & \ $A, C,D \to W$ (Acc. $\%$) &  $\epsilon$   & $A, C, W \to D$(Acc.)& $\epsilon$\\
					\hline
					\multirow{5}{*}{AlexNet}        & FedAvg          & $90.5\pm0.1$    & - & $96.8\pm0.1$& - \\
					&DP-FedAvg\     & $28.1\pm0.7$ &  46.6&  $48.2\pm0.8 $& 47.1\\
					&DP-FedSGD\     & $32.6\pm 0.9$ &  4.1& $48.3\pm0.9$&4.0\\
					&DP-FedSGD\     & $75.2\pm0.5$ & 12.4 &  $83.7\pm0.6$ & $7.9$\\
					&\textit{kNN-DPFL} ($\sigma=15$, Ours) &$\bf{75.4}\pm \bf{0.3}$ & \bf{3.9} & $\bf{84.3}\pm\bf{0.3}$& \bf{3.7}\\
					\hline
					\multirow{3}{*}{ResNet50}        & FedAvg         & $96.5\pm0.1$    &- &$ 97.8\pm0.1$& - \\
					&DP-FedSGD    & $25.8\pm0.6$&  4.0 &   $42.7\pm0.5$ & 3.9 \\
					&\textit{kNN-DPFL} ($\sigma=25$, Ours)      & $\bf{86.3}\pm\bf{0.4}$&   \bf{2.8}  & $\bf{91.9}\pm\bf{0.2} $   & \bf{2.0} 
					\\         
					\hline
				\end{tabular}
		}
	\end{center}
	\vspace{-2mm}
	\caption{\textbf{Instance-level DP on Office-Caltech.} 
	}
	\label{tab:result_office}
	\vspace{-2mm}
\end{table}

\begin{table} [t]
	\begin{center}
		\vspace{-1mm}
		\small
		\centerline{
				\begin{tabular}{c@{\hskip .01mm}cccccc}
					\toprule
					& Clipart (Acc.$\%$) &  $\epsilon$   &  Painting (Acc.$\%$)& $\epsilon$ &Real (Acc.$\%$) & $\epsilon$\\
					\hline
					FedAvg & $81.8\pm 0.2$ &  - &  $72.8\pm0.2$ & -&  $82.0\pm0.3$&-\\
					DP-FedSGD   & $44.2\pm0.2$ &  4.4& $ 42.6\pm0.3$ & \bf{4.6} & $39.1\pm0.6$&4.3\\
					DP-FedSGD    & $55.6\pm0.2$ & $11.6$ &  $60.0\pm 0.6 $& 14.6& $55.1\pm0.6$&11.9\\
					\textit{kNN-DPFL} (Ours)       &$\bf{55.8}\pm \bf{0.6}$ &\bf{4.4} & $\bf{61.2} \pm\bf{0.8}  $   & 4.7 & $\bf{55.5}\pm \bf{0.7}$&  \bf{4.2} \\
					\hline
				\end{tabular}
		}
	\end{center}
	\vspace{-4mm}
	\caption{\textbf{Instance-level DP on DomainNet.}  Total number of local agents is 5. We set $\delta=10^{-4}$.}
	\label{tab:cls3}
	\vspace{-5mm}
\end{table}
\textbf{Data-dependent privacy analysis:}
In Table~\ref{tab:result_svhn}, we provide the data-dependent DP ($\epsilon^*$) for Digit and MNIST datasets. CelebA is a multi-label dataset where the improved privacy bound is not applicable. We observe that the privacy guarantee improves significantly (e.g., $\epsilon=3.6$ to $\epsilon^*=2.7$) in Digit dataset, while no amplification in the MNIST dataset. It is because the Digit dataset contains more agents, which allows a larger margin among voting, while in the MNIST dataset, each agent can only observe samples from $6$ digits, which decreases the margin by limiting the effective voter for each unlabeled data. Our results suggest that we can amplify the privacy guarantee whenever there is a high consensus among voters.

\vspace{-1em}
\subsection{Instance-level DP Evaluation}
\vspace{-0.5em}
We investigate the instance-level DP using datasets Office-Caltech10 \citep{caltech} and DomainNet \citep{peng2019moment} from the multi-source domain adaptation literature.



\textbf{Hyperparameters:} The DP-FedSGD method provides the DP baseline where we use mostly the same parameters as \citet{abadi2016deep} but we tune the noise scale based on the fixed privacy budget. The hyperparameters of \textit{kNN-DPFL} consist of the choice of $k$ neighbors, the noise scale $\sigma$, and the number of unlabeled data to be labeled $Q$. We set $k$ to be the $5\%$ of the local data size (i.e., each agent returns the noisy top-$5\%$ neighbors’ predictions). This choice can be further improved by setting $k$ individually, e.g., setting $k$ to be smaller if the local agent observes the domain gap is large. 
In each experiment, we split $70\%$ data from the server domain as the public available unlabeled data, which is also the data to be labeled for \textit{kNN-DPFL}, while the remaining $30\%$ data is used for testing. 


\textbf{Office-Caltech Evaluation:} 
Office-Caltech consists of data from four domains: Caltech (C), Amazon (A), Webcam(W) and DSLR (D). We iteratively pick one domain as the server domain each time and the rest ones are for local agents (e.g., in $A, C,D \to W $, W is treated as the server). For \textit{kNN-DPFL}, we instantiate the public feature extractor using the network backbone without the classifier layer.

To investigate the scalability to high-capacity models, we explore the privacy-utility tradeoffs under two network backbones in Table~\ref{tab:result_office}. Both AlexNet and Resnet50 are Imagenet pre-trained.   
We observe: (1) DP-FedSGD degrades when the backbone changes from the light load AlexNet to the heavy load ResNet50, while ours is improved by $10\%$. It is because larger model capacity leads to more sensitive response to gradient clipping or noise injection, which has been surveyed in~\citet{abadi2016deep}. In contrast, our \textit{kNN-DPFL} avoids the gradient operation by label aggregation and can still benefit from the larger model capacity. Again, our method achieves consistently better utility-privacy trade-off as maintaining same privacy cost and can achieve significantly better utility, or maintaining same utility and can achieve much lower privacy cost. 


\textbf{DomainNet Evaluation:}
DomainNet contains 0.6 million images ranging from six domains: Clipart, Painting, Real, Quickdraw, Infograph, and Sketch. Given that it is a challenging dataset even for the non-private setting \citep{peng2019}, we only report results on cases where the server is iteratively chosen from Clipart, Painting, and Real. In each case, the remaining domain data are assigned to five local agents, respectively. 

Table~\ref{tab:cls3} compares our \textit{kNN-DPFL} method with DP-FedSGD. We observe that when the privacy cost $\epsilon$ is aligned close, our method outperforms DP-FedSGD by more than $10\%$ in accuracy gain across all the three cases. When the accuracy is aligned close, our method saves more than $60\%$ privacy cost, showing consistent advantages over DP-FedSGD.

\vspace{-1em}
\subsection{Ablation Study}
\vspace{-0.5em}
We investigate the agent-level privacy-utility trade-off with respect to $\#$agents and the volume of local data.  The overall privacy budget is fixed to $(\epsilon=5, \delta = 10^{-3})$. MNIST is used for generality and simplicity. We randomly pick $1000$ testing data as the unlabeled server data and the remaining $9000$ data for testing. We adopt the network backbone used in \citep{abadi2016deep} for all methods.

\textbf{Effect of \textit{Data per Agent}:} We fix the number of agents to $100$ and range the number of data per agent from $\{50, 100, 200, 400, 600\}$. By only relaxing the ``data per agent'' factor, we fairly tune the other privacy parameters for all the methods to its maximized performance. In Figure \ref{fig:ablation} (a), as ``data per agent'' increases, all the methods improves as the overall dataset volume increases. Our method achieves consistently higher accuracy over DP-FedAvg. The failure cases for both methods are when ``data per agent'' is below 50, which cannot ensure the well-trained local agent models. Label aggregation over such weak local models results in failure or sub-optimal performance.

\textbf{Effect of \textit{Number of Agents}:}  Following \citep{dp_fl}, we vary the number of agents $N \in \{ 50, 100, 200, 400, 800\}$ and each agent has exactly 600 data, where data samples are randomly duplicated when $N \in \{200, 400, 800\}$. We conduct grid search for each method to obtain optimal hyper-parameters. In Figure \ref{fig:ablation} (b), our method shows clear performance advantage over DP-FedAvg. We also see DP-FedAvg gradually approaches our method as the number of agents increases. 
%

\vspace{-2mm}
\section{Conclusions}

In this work, we propose two voting-based DPFL approaches, \textit{AE-DPFL} and \textit{kNN-DPFL}, under two privacy regimes: agent-level and instance-level DP. We substantially investigate the real-world challenges of DPFL and demonstrate the advantages of our methods over gradient aggregation based DPFL methods on utility, convergence, reliance on network capacity, and communication cost. Extensive empirical evaluation shows that our methods improve the privacy-utility trade-off in both privacy regimes.


\bibliography{example_paper}
\bibliographystyle{iclr2021_conference}
\newpage
\appendix
\onecolumn
In the appendix, we start by introducing other properties of differential privacy in Sec~\ref{section: rdp}. Then the discussions of challenges for gradient-based DPFL is discussed in Sec~\ref{sec:challenge}. In Sec~\ref{sec: data_depend} we provide the proof of our data-dependent privacy analysis, and more experimental details are provided in Sec~\ref{sec: exp}.
\section{Other properties of differential privacy}\label{section: rdp}

\begin{definition}[Renyi Differential Privacy \citep{mironov2017renyi}]
We say a randomized algorithm $\cM$ is $(\alpha, \epsilon(\alpha))$-RDP with order $\alpha \geq 1$ if for neighboring datasets $D, D'$,
$$
\mathbb{D}_{\alpha}(\cM(D)||   \cM(D')):= \frac{1}{\alpha-1}\log \mathbb{E}_{o \sim \cM(D')}\bigg[ \bigg( \frac{\pr[\cM(D)=o]}{\pr[\cM(D')=o]}\bigg)^\alpha \bigg]\leq \epsilon(\alpha).
$$
\end{definition}

RDP inherits and generalizes the information-theoretical properties of DP.
\begin{lemma}[Selected Properties of RDP \citep{mironov2017renyi}]\label{lem: property of RDP}
If $\cM$ obey $\epsilon_{\cM}(\cdot)$-RDP, then 
\begin{enumerate}
    \item[1.] [Indistinguishability] For any measurable set $S\subset \text{Range}(\cM)$, and any neighboring $D,D'$ 
$$e^{-\epsilon(\alpha)}\pr[\cM(D')\in S]^{\frac{\alpha}{\alpha-1}} \leq \pr[\cM(D)\in S] \leq e^{\epsilon(\alpha)}\pr[\cM(D')\in S]^{\frac{\alpha-1}{\alpha}}.$$ 
\item[2.] [Post-processing] For all function $f$, $\epsilon_{f\circ\cM}(\cdot) \leq \epsilon_{\cM}(\cdot).$
\item[3.] [Composition] $\epsilon_{(\cM_1,\cM_2)}(\cdot) = \epsilon_{\cM_1}(\cdot) + \epsilon_{\cM_2}(\cdot).$
\end{enumerate}
\end{lemma}

This composition rule often allows for tighter calculations of $(\epsilon,\delta)$-DP for the composed mechanism than the strong composition theorem in~\citep{kairouz2015composition}.
Moreover, we can covert RDP to $(\epsilon,\delta)$-DP for any $\delta >0$ using:
\begin{lemma}[From RDP to DP] If a randomized algorithm $\cM$ satisfies $(\alpha,\epsilon(\alpha))$-RDP, then $\cM$ also satisfies $(\epsilon(\alpha)+\frac{\log(1/\delta)}{\alpha-1},\delta)$-DP for any $\delta \in (0,1)$. \label{lem: rdp2dp}
\end{lemma}

\section{More Discussions of Challenges for Gradient-Based FL\label{app: convergence}}
\begin{definition}
	A function $\ell$ is Lipschitz continuous with constant $G>0$, if $$|\ell(x) - \ell(y)| \leq G ||x-y||_2$$ for all $x,y$.
\end{definition}

\begin{proposition}
	Let the objective function of agents $f_1,..., f_N$ obeys that $f_i$ is piecewise linear (which implies that the global objective $F = \frac{1}{N}\sum_{i=1}^N f_i$ is piecewise linear) and $G$-Lipschitz. Let $\eta$ be the learning rate taken by individual agents. Then the outer loop FedAvg update is equivalent to $\theta^+ = \theta - E\eta g$ for some $g\in \R^d$, where 
	(a) $g = \nabla F(\theta)$ if $\theta$ is in the $\nu$ interior of the linear region of $f_1,...,f_N$ and  $E < \nu / (\eta G)$; (2) $g$ is a Clarke-subgradient \footnote{Clarke-subgradient is a generalization of the subgradient to non-convex functions. It reduces to the standard (Moreau) subgradient when $F$ is convex.} of  $F$ at $\theta$, if $\theta$ is on the boundary of at least two linear regions and at least $\nu$ away in Euclidean distance from another boundary and $E<\nu / (\eta G)$; (c) otherwise, we have that $\|g - \nabla F(\theta)\|_2 \leq E \eta G$. Moreover, statement (c) is true even if we drop the piecewise linear assumption.
\end{proposition}
\begin{proof}
	For the Statement (a), observe that for all $\theta'$ such that $\|\theta' - \theta\| \leq \nu$ neighborhood, we have that $\nabla f_i(\theta') = \nabla f_i(\theta)$.  When $E< \nu / (\eta G)$, the cumulative gradients of agent $i$ is equal to $E \nabla f_i(\theta)$.   For Statement (b), notice that the Clarke subdifferential at $\theta$ is the convex hull of the one-sided gradient, thus as we move along the negative gradient direction in the inner loop, we enter and remains in the linear region. Thus the update direction is 
	$$\frac{1}{N} \left(\sum_{i \text{ s.t. } f_i \text{ is differentiable at }\theta} E \eta \nabla f_i(\theta) +  \sum_{i \text{ s.t. } f_i \text{ is not differentiable at }\theta} \eta g_i + (E-1)\nabla f_i(\theta-\eta g_i)\right)$$
	for all $g_i$ such that it is a Clarke-subgradient of $f_i$ it can be written as a convex combination. The proof is complete by observing that the $1/N\sum_i$ is also a convex combination and by multiplying and dividing by $E$.
	Statement (c) is a straightforward application of the Lipschitz property which says that $E$ steps can at most get you away for $\eta EG$ and clearly piecewise linear assumption is not required.
\end{proof}
This proposition says that in almost all $\theta$, increasing $E$ has the effect of increasing the learning rate of the subgradient ``descent'' method for piecewise linear objective functions; and increasing the learning rate of an approximate gradient method in general for Lipschitz objective functions. It is known that for the family of $G-$Lipschitz function supported on a $B$-bounded domain, any Krylov-space method \footnote{One that outputs a solution in the subspace spanned by a sequence of subgradients.} has a rate of convergence that is lower bounded by $O(BG/\sqrt{T})$ if running for $T$ iterations. A close inspection of the lower bound construction reveals that the worst-case problem is $\min_{\theta\in\R^T} \max_i \theta_i  + \|\theta\|^2$, namely, a regularized piecewise linear function.  This is saying that the variant of FedAvg that aggregates only the loss-function part of the gradient or projects only when synchronizing essentially requires $\Omega(1/\alpha^2)$ rounds of outer loop iterations (thus communication) in order to converge to an $\alpha$ stationary point, i.e., increasing $E$ does \emph{not} help, even if no noise is added.

\section{Data-dependent Privacy Analysis}~\label{sec: data_depend}
\subsection{Privacy Analysis}
\begin{theorem}[Restatement of Theorem~\ref{thm:privacy}]
	Let \textit{AE-DPFL} and \textit{kNN-DPFL} answer $Q$ queries with noise scale $\sigma$. For agent-level protection, both algorithms guarantee $(\alpha, Q\alpha/(2\sigma^2))$-RDP for all $\alpha\geq 1$. 
	For instance-level protection, \textit{AE-DPFL} and \textit{kNN-DPFL} obey $ (\alpha, Q\alpha/\sigma^2)$ and
	$ (\alpha, Q\alpha/(k\sigma^2))$-RDP respectively.
\end{theorem}
\begin{proof}
	In \textit{AE-DPFL}, for query $x$, by the independence of the noise added, the noisy sum is identically distributed to $\sum_{i=1}^{N} f_i(x) + \cN(0, \sigma^2)$.
	Adding or removing one data instance from will change $\sum_{i=1}^{N} f_i(x)$ by  at most $\sqrt{2}$ in L2. The Gaussian mechanism thus satisfies $(\alpha, \alpha s^2/2\sigma^2)$-RDP on the instance-level   for all $\alpha\geq 1$ with an L2-sensitivity $s = \sqrt{2}$.  This is identical to the analysis in the original PATE~\citep{papernot2018}.
	
	For the agent-level, the L2 and L1 sensitivities are both $1$ for adding or removing one agent. 
	
	In \textit{kNN-DPFL}, the noisy sum is identically distributed to $\frac{1}{k}\sum_{i=1}^N\sum_{j=1}^k y_{i,j} +\cN(0, \sigma^2)$. The change of adding or removing one agent will change the sum by at most $1$, which implies the same L2 sensitivity and same agent-level protection as \textit{AE-DPFL}.  The $L2$-sensitivity from adding or removing one instance, on the other hand changes the score by at most $\sqrt{2/k}$ in L2 due to that the instance being replaced by another instance, this leads to an an improved instance-level DP that reduces $\epsilon$ by a factor of $\sqrt{\frac{k}{2}}$.
	
	The overall RDP guarantee follows by the composition over $Q$ queries.  The approximate-DP guarantee follows from the standard  RDP to DP conversion formula $\epsilon(\alpha) + \frac{\log(1/\delta)}{\alpha-1}$ and optimally choosing $\alpha$.
\end{proof}


\subsection{Improved accuracy and privacy with  large margin}

Let $f_1,...,f_N: \cX\rightarrow \triangle^{C-1}$ where $\triangle^{C-1}$ denotes the probability simplex --- the soft-label space.  Note that both algorithms we propose can be viewed as voting of these teachers which outputs a probability distribution in $\triangle^{C-1}$.
First let us define the margin parameter $\gamma(x)$ which measures the difference between the largest and second largest coordinate of $\frac{1}{N}\sum_{i=1}^N f_i(x)$. 

\begin{lemma}
	Conditioning on the teachers, for each public data point $x$,  the noise added to each coordinate is drawn from $\cN(0, \sigma^2/N^2)$, 
	then with probability $\geq 1-C\exp\{-N^2\gamma(x)^2/8\sigma^2\}$, the  privately released label matches the majority vote without adding noise.
\end{lemma}
\begin{proof}
	The proof is a straightforward application of Gaussian tail bounds and a union bound over $C$ coordinates. Specifically, $\P[Z_{j^*}< -\gamma(x)/2] \leq e^{-\frac{N^2\gamma(x)^2 }{8\sigma^2}}$ for the argmax $j^*$. For $j\neq j^*$,  $\P[Z_{j} >  \gamma(x)/2] \leq e^{-\frac{N^2\gamma(x)^2 }{8\sigma^2}}$.  By a union bound over all coordinates $C$, we get that there perturbation from the boundedness is smaller than $\gamma(x)/2$,  which implies correct release of the majority votes.
\end{proof}
This lemma implies that for all public data point $x$ such that $\gamma(x) \geq \frac{2\sqrt{2\log(C/\delta)}}{N}$, the output label matches noiseless majority votes with probability exponentially close to $1$.

Next we show that for those data point $x$ such that $\gamma(x)$ is large, the privacy loss for releasing $\argmax_j [\frac{1}{N}\sum_{i=1}^N f_i(x)]_j$ is exponentially smaller. The result is based on the following privacy amplification lemma that is a simplification of Theorem~6 in the appendix of \citep{papernot2018}.
\begin{lemma}\label{lem:amplification}
	 Let $\cM$ satisfy $(2\alpha,\epsilon)$-RDP, and there is a singleton output that happens with probability $1-q$ when $\cM$ is applied to $D$. Then for any $D'$ that is adjacent to $D$, Renyi-divergence
	 $$
	 D_{\alpha}(\cM(D)\|\cM(D')) \leq -\log(1-q) + \frac{1}{\alpha-1}\log(1 +  q^{1/2}(1-q)^{\alpha-1} e^{(\alpha-1) \epsilon}).
	 $$
\end{lemma}
\begin{proof}
	Let $P,Q$ be the distribution of of $\cM(D)$ and $\cM(D')$ respectively and $E$ be the event that the singleton output is selected.
	\begin{align*}
	\E_Q[(dP/dQ)^\alpha]  &=  \E_Q[(dP/dQ)^\alpha | E] \P_Q[E] +   \E_Q[(dP/dQ)^\alpha \mathbf{1}(E^c) \\
	 &\leq  (1-q)(\frac{1}{1-q})^{\alpha}  +  \sqrt{\E_Q[(dP/dQ)^(2\alpha)]} \sqrt{\E_Q[\mathbf{1}(E^c)^2]}\\
	 &\leq   (1-q)^{-(\alpha-1)} + q^{1/2} e^{(2\alpha-1)\epsilon/2} = (1-q)^{-(\alpha-1)} \left(  1 + (1-q)^{\alpha-1} q^{1/2} e^{\frac{2\alpha-1}{2}\epsilon}\right)
	 \end{align*}
	 The first part of the second line uses the fact that event $E$ is a singleton with probability larger than $1-q$ under $Q$ and the probability is always smaller than $1$ under $P$. The second part of the second line follows from  Cauchy-Schwartz inequality. The third line substitute the definition of $(2\alpha,\epsilon)$-RDP. Finally, the stated result follows by the definition of the Renyi divergence.
\end{proof}

\begin{theorem}[Restatement of Theorem~\ref{thm: data_dependent}]
	The mechanism that releases $\argmax_j [\frac{1}{N}\sum_{i=1}^N f_i(x) +  \cN(0, (\sigma^2/N^2)  I_C)]_j$ obeys $(\alpha,\epsilon)$-data-dependent-RDP, where
		$$
	\epsilon \leq 2Ce^{-\frac{N^2\gamma(x)^2}{8\sigma^2}} +  \frac{1}{\alpha-1}\log\left(1 +   e^{\frac{(2\alpha-1)\alpha s}{2\sigma^2} - \frac{N^2\gamma(x)^2}{8\sigma^2} + \log C/2}\right),
	$$
	where $s=1$ for AE-DPFL with the agent-level DP, and $s=2/k$ for KNN-DPFL with the instance-level DP.
\end{theorem}
\begin{proof}
	The proof involves substituting $q= Ce^{-\frac{N^2\gamma(x)^2}{8\sigma^2}} $ from Lemma~\ref{lem: majorite_vote} into  Lemma~\ref{lem:amplification} and use the fact that $\cM$ satisfies the RDP of a Gaussian mechanism from the RDP's post-processing lemma.
	The expression bound is simplified for readability using  $-\log(1-x)< 2x$ for all $x>-0.5$ and that $(1-q)^{\alpha-1}\leq 1$.
\end{proof}
As we can see, when given teachers that are largely in consensus, the (data-dependent) privacy loss exponentially smaller.

 \section{Datasets and Models\label{app:exp}}\label{sec: exp}
  Here we provide full details on the datasets and models used in our experiments.

 \textbf{Hyperparameters.}  For DP-FedAvg, the hyperparameters include agent sampling probability $q$, the noise parameter $\sigma$, the clipping threshold $S$. We do a grid search on all hyperparameters,  and observe $(S = 0.08, \sigma = 0.06)$ works best for the simple CNN  (used in ablation study) and AlexNet.  The choice of $q$ depends on the number of agents and the task complexity.  A smaller $q$ implies a stronger privacy guarantee and a larger variance.  We set $q=0.05$ for Digit dataset and $q=0.04$ for CelebA. The number of local iterations $ E $ is another consideration. We empirically observe $E = 20$ achieves beset trade-offs between privacy and accuracy.  
 For all experiments, the learning rate is $0.015$, and we decay the learning rate through communication rounds, which leads to better performance compared to the original implementation in~\citep{dp_fl}.

 For DP-FedSGD, we train each local model using Noisy SGD~\citep{abadi2016deep}, where the privacy parameters include batch size, the clipping threshold $S$, and the noisy scale $\sigma$.  After a grid search, we use a batch size of $16$ for Caltech dataset and $32$ for DomainNet.  We set the clipping threshold $S$ to $0.08$ and tune the noisy scale based on a fixed privacy budget. To amplify the privacy guarantee of DP-SGD using SMC, we set the number of local iteration $E=1$. 
 
 \textbf{Dataset.} We provide detailed information datasets here.
 For Office-Caltech and DomainNet-fruit, we provide the number of images in each domain. An overview of DomainNet with seven selected fruit classes is depicted in Figure~\ref{fig: domainnet}. 
 

 \textbf{Details of CelebA Datasets Evaluation}
 For DP-FedAvg, we set $(S=0.08, \sigma = 0.06, q= 0.04)$. Note that the global sensitivity depends on the number of attributes, in which we use the same clipping technique in \citep{private_kNN} to restrict each agent's prediction clipped to $\tau$ attributions. We set $\tau =4, \sigma = 50$ for AE-DPFL. We apply AlexNet for all methods in this evaluation.
 
 
 \textbf{DomainNet Evaluation:} 
We set $\sigma=35$ for \textit{kNN-DPFL}. Since \textit{kNN-DPFL} can scale to high-capacity model while DP-FedSGD cannot, we use ResNet50 for \textit{kNN-DPFL} and DP-FedSGD is trained with AlexNet. 
	\begin{table}
		\centering
	\begin{tabular}{cccccccc}\\\toprule  
	 Splits & Clipart & Infograph& Painting & Quickdraw& Real& Sketch &Total \\\midrule
		Total& 938 & 1585 & 2274&3500 &3282  &1312  &12891 \\  
	\end{tabular}
		\caption{DomainNet with seven classes}
\end{table}
\begin{table}
	\centering
	\begin{tabular}{cccccc}\\\toprule  
		Splits & Amazon & Dslr & Webcam &Caltech&Total \\\midrule
		Total& 958 & 157 & 295 & 1123 & 2533 \\  
	\end{tabular}
	\caption{Office-Caltech10}
\end{table}  

 \begin{figure}[t]
\centering	
\includegraphics[width = 120mm]{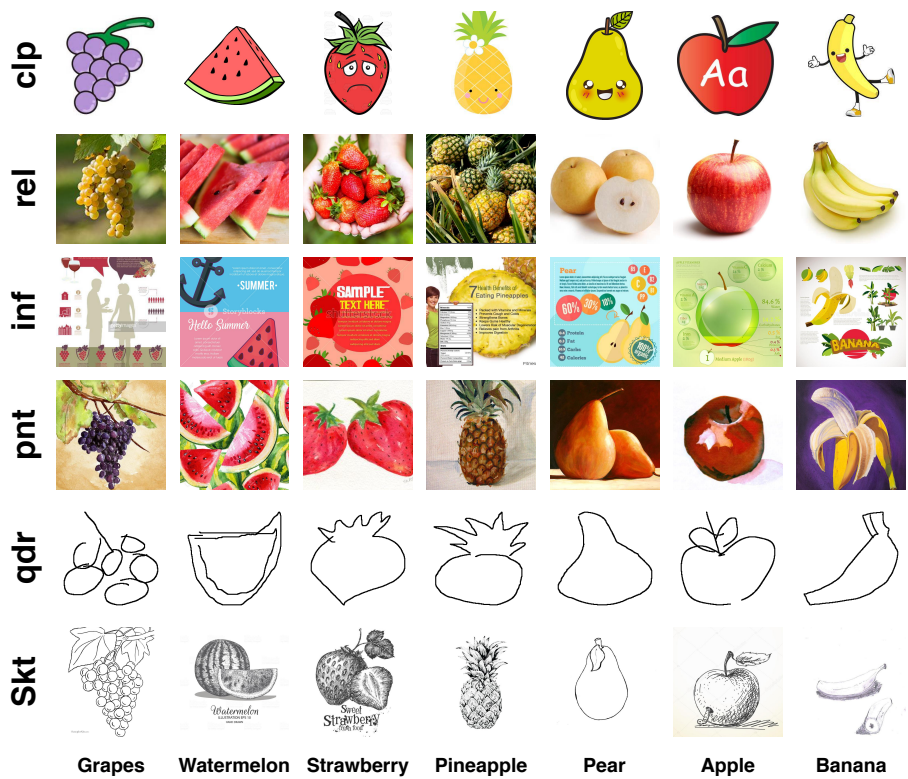}
\caption{An overview of DomainNet dataset with seven selected fruit classes.}
\label{fig: domainnet}
\vspace{-4mm}
\end{figure}

\end{document}